\documentclass{article}
\usepackage{fullpage}
\usepackage{amsthm}
\usepackage{natbib}
\usepackage[fourier]{stroop}
\usepackage{euflag}
\title{Two derivations of Principal Component Analysis \\ on datasets of distributions}
\author{Vlad Niculae \quad (\texttt{v.niculae@uva.nl})\\ Language Technology
Lab, University of Amsterdam, The Netherlands}

\newcommand\var[1]{\mathsf{#1}}
\DeclareMathOperator{\tr}{tr}

\begin{document}
\date{}
\maketitle
\begin{abstract}
In this brief note, we formulate Principal Component Analysis (PCA)
over datasets consisting not of points but of distributions, characterized by
their location and covariance. Just like the usual PCA on points can be
equivalently derived via a variance-maximization principle and via a
minimization of reconstruction error, we derive a closed-form solution for
distributional PCA from both of these perspectives.
\end{abstract}

\section{Introduction}

Most commonly in data science we are concerned with datasets that consist of
points \(\{x_1, \ldots, x_n\} \subset \bbR^d \).  In this note we focus on
datasets of \emph{random variables}
\(\{\var{x}_1, \ldots, \var{x}_n\}\)
where each \(\var{x}_i\) has a probability distribution
summarized by its mean and variance
\[
\bbE[\var{x}_i] = \mu_i,\qquad
\bbV[\var{x}_i] = \Sigma_i.
\]
This scenario fully subsumes the standard point dataset in the limit of
\(\Sigma_i \to 0\), but allows us to further model situations such as:
\begin{itemize}
\item Uncertainty or measurement noise (\ie, inherent variability of the
\(\var{x}_i\)s),
\item Hierarchical data (\eg, psychometric data where each \(\var{x}_i\) is a study
participant for whom several measurements are taken).
\end{itemize}

In this work we extend the usual pointwise PCA to distributional data.
We first recap the definition and derivation of PCA. Then, we show
two different derivations of its distributional counterpart.

\section{Background}

PCA \citep[section 20.1]{pml1Book} is a workhorse of statistical analysis, data science,
and visualization. It
is a dimensionality reduction technique that summarizes a (point) dataset by
linearly transforming it to the most important dimensions of variability.
There are two typical ways to define \emph{most important}, and it turns out
they both lead to the same result.
For this section we assume
a centered point dataset \(\{x_1, \ldots, x_n\}\), i.e., \(\sum_i x_i = 0\).
\footnote{PCA is typically defined after centering, but in some scenarios (\eg,
high-dimensional sparse data) centering is sometimes skipped. While centering
is important for some statistical interpretation of the method, it makes no
difference for our derivation.}

\paragraph{Directions of maximal variance}

One road toward PCA starts with the question: what is the direction
\(u\) that maximizes the variance of the projected dataset?
In other words, we seek:
\begin{equation}\label{eq:pca_variance}
\arg\max \left\{ \sum_i (u^\top x_i)^2 : u \in \bbR^d, \|u\|=1 \right\}.
\end{equation}
This is because \(z_i = u^\top x_i\) is the projection of \(x_i\) along the direction
of \(u\), and since we assume the \(x_i\)s are centered then so are the
\(z_i\)s, and so the objective of \cref{eq:pca_variance} is the empirical
variance of the \(z_i\)s.

We may rewrite the objective of \cref{eq:pca_variance} as \(u^\top S u \) where
\(S = \sum_i x_i x_i^\top \),
and therefore we recognize that the solution of \cref{eq:pca_variance} is the
eigenvector of \(S\) corresponding to the largest eigenvalue.
This view readily extends to seeking the top-k principal components \(u_1,
\ldots, u_k\) by requiring additional orthogonal constraints, \ie, \(U^\top U =
I\), and the solution is likewise given by the top-k eigenvectors of \(S\).

\paragraph{Minimizing the reconstruction error}
If we view \(z_i = u^\top x_i\) as a 1-d encoded representation of \(x_i\),
we can map \(z_i\) back into \(\bbR^n\) as the vector \(z_i u\) in the span of
\(u\). This process will lose information.
We may then ask the question: which direction \(u\) minimizes the
reconstruction error of this encoding-decoding process? Or,
\begin{equation}
\arg\min \left\{ \sum_i \| x_i - uu^\top x_i \|^2 : u \in \bbR^d, \|u\|=1
\right\}.
\end{equation}
Denote \(Q=uu^\top\) and remark \(Q\) is a projection matrix, thus idempotent
and so is \(I-Q\). Then
\begin{equation}\label{eq:minerr}
\begin{aligned}
\sum_i \|x_i - Qx_i\|^2
&= \sum_i \|(I-Q)x_i \|^2\\
&= \sum_i x_i^\top (I-Q) x_i \\
&= -\sum_i x_i^\top (uu^\top) x_i + \text{const}\\
&= -u^\top S u + \text{const}, \\
\end{aligned}
\end{equation}
where the last step uses the same rearranging of the dot product as in the
paragraph above.

So, minimizing the reconstruction error, or maximizing projected variance, are
equivalent views that lead to the same principal component solution.

\section{Deriving distributional PCA}

We propose the following formulation for PCA over a dataset of random variables
\(\{\var{x}_1, \ldots, \var{x}_n\}\):

\begin{definition}[Distributional PCA]\label{def:dpca}
Given a dataset of random variables, denoted \(\{\var{x}_1, \ldots,
\var{x}_n\}\),
with means \(\mu_i\) and covariance matrices \(\Sigma_i\), the principal
components of this dataset are the leading eigenvectors of the matrix:

\[ \sum_i \mu_i \mu_i^\top + \Sigma_i. \]
\end{definition}

We shall give two justifications of this definition.

\begin{proposition}
Distributional PCA, as in \cref{def:dpca}, maximizes the expected projected
variance:
\begin{equation}
\arg\max \left\{
\bbE_{\var{x}_1, \ldots, \var{x}_n} \left[ \sum_i (u^\top \var{x}_i)^2\right]:
 %\sum_i \bbE_{\var{x}_i}\left[ (u^\top \var{x}_i)^2\right]:
u \in \bbR^d, \|u\|=1 \right\}.
\end{equation}
\end{proposition}

\begin{proof}
Rearranging and using linearity, we may rewrite the objective as
\[
\begin{aligned}
\bbE\left[\sum_i (u^\top \var{x}_i)^2 \right]
&= \bbE\left[\sum_i u^\top(\var{x}_i \var{x}_i^\top) u\right] \\
&= \sum_i u^\top\left(\bbE[ \var{x}_i \var{x}_i^\top]\right) u\\
&=
u^\top\left(
\sum_i \mu_i \mu_i^\top + \Sigma_i
\right) u.
\end{aligned}
\]
\end{proof}

\begin{proposition}
%If we further assume that \(\var{x}_i\) are normal distributions, then
Distributional PCA, as in \cref{def:dpca}, minimizes the total squared
2-Wasserstein reconstruction error under the linear projection:
\begin{equation}
\arg\min \left\{
    \sum_i W_2^2(\var{x}_i, uu^\top \var{x}_i)
    : u \in \bbR^d, \|u\|=1
\right\}.
\end{equation}
\end{proposition}

To prove this result, we need the following lemma:
\begin{lemma}[Masarotto]
Let \(\var{x}\) be a random variable with mean \(\mu\) and variance
\(\Sigma\), and \(Q\) be a projection matrix.
Then
\[
W_2^2(\var{x}, Q\var{x}) = \|\mu - Q\mu\|^2 + \tr((I-Q)\Sigma).
\]
\end{lemma}
\begin{proof} (of the lemma).
This is a slight extension of the unnumbered result given by
\citet{masarotto2019procrustes} in their section 5.
First, we use the translation property of \(W_2\)
\citep[Remark 2.19]{cot}
to reduce the problem to a
distance between zero-mean measures:
\[
W_2^2(\var{x}, Q\var{x}) = \|\mu - Q\mu\|^2 +
W_2^2(\var{\bar{x}}, Q\var{\bar{x}})
\]
where \(\var{\bar{x}}=\var{x}-\mu\).
If \(\alpha\) is the probability measure associated with \(\var{\bar{x}}\),
then \(\beta = Q_\sharp \alpha\) is the probability measure of
the pushforward \(Q\var{\bar{x}}\).
Since \(Q\) is a projection matrix, it is symmetric positive semidefinite and
therefore it is the gradient of a convex mapping \(x \to \frac{1}{2} x^\top Q x\)
By Brenier's theorem \citep[Remark 2.24]{cot}, Q is the optimal transport plan
between \(\alpha\) and \(\beta\). This implies
\[
\begin{aligned}
W_2^2(\var{\bar{x}}, Q\var{\bar{x}})
&= \int_x
d\alpha~
\|x - Qx\|^2  \\
&= \int_x d\alpha~x^\top (I-Q) x \\
%&= \int_x
%d\alpha~
%x^\top x
%- x^\top Q x
%\qquad (\text{since } Q^\top Q = Q)
%\\
&= \int_x
d\alpha~\tr \left( (I-Q) xx^\top \right)\\
&= \tr \left( (I-Q) \Sigma \right).
\end{aligned}
\]
\end{proof}

\begin{proof} (of the proposition).
Let \(Q=uu^\top\) denote the projection operator onto the span of \(u\).
Applying the lemma,
\begin{equation}
\begin{aligned}
\sum_i W_2^2(\var{x}_i, Q\var{x}_i) &=
\left( \sum_i \mu_i^\top (I-Q) \mu_i
+ \tr\left((I-Q)\Sigma_i\right) \right)\\
&= -\sum_i \tr\left(Q(\mu_i\mu_i^\top + \Sigma_i)\right) + \text{const}\\
&= -u^\top \left(\sum_i \mu_i \mu_i^\top + \Sigma_i\right) u + \text{const}.
\end{aligned}
\end{equation}
\end{proof}

We have thus shown that distributional PCA can also be viewed equivalently from
a variance-maximization and error-minimization angle, just like usual pointwise
PCA. In addition, in the limit of all \(\Sigma_i \to 0\), we recover usual PCA.
Finally we remark that while we use a single principal component in the above
derivations, everything holds for \(k\) orthogonal principal components as well.

\section{Discussion}

\paragraph{Visualization.}
To demonstrate how distributional PCA works, we construct a dataset with four
Gaussian random variables. Their locations are
\(
\mu_1 = (-0.5, -2),
\mu_2 = (0.5, -1),
\mu_3 = (-0.5, 0),
\mu_4 = (-0.5, 1),
\)
and their covariances are all equal to \(\Sigma=\diag(1, 0.5)\).
\Cref{fig:samples} shows the principal component direction obtained by
performing the usual PCA on the four means, performing distributional PCA, and
performing PCA on a dataset obtained by drawing 1000 samples from each of the
four distributions. Our proposed formula indeed characterizes the limit case
of sampling from the distributional dataset.
\begin{figure}\centering
\includegraphics[width=.6\textwidth]{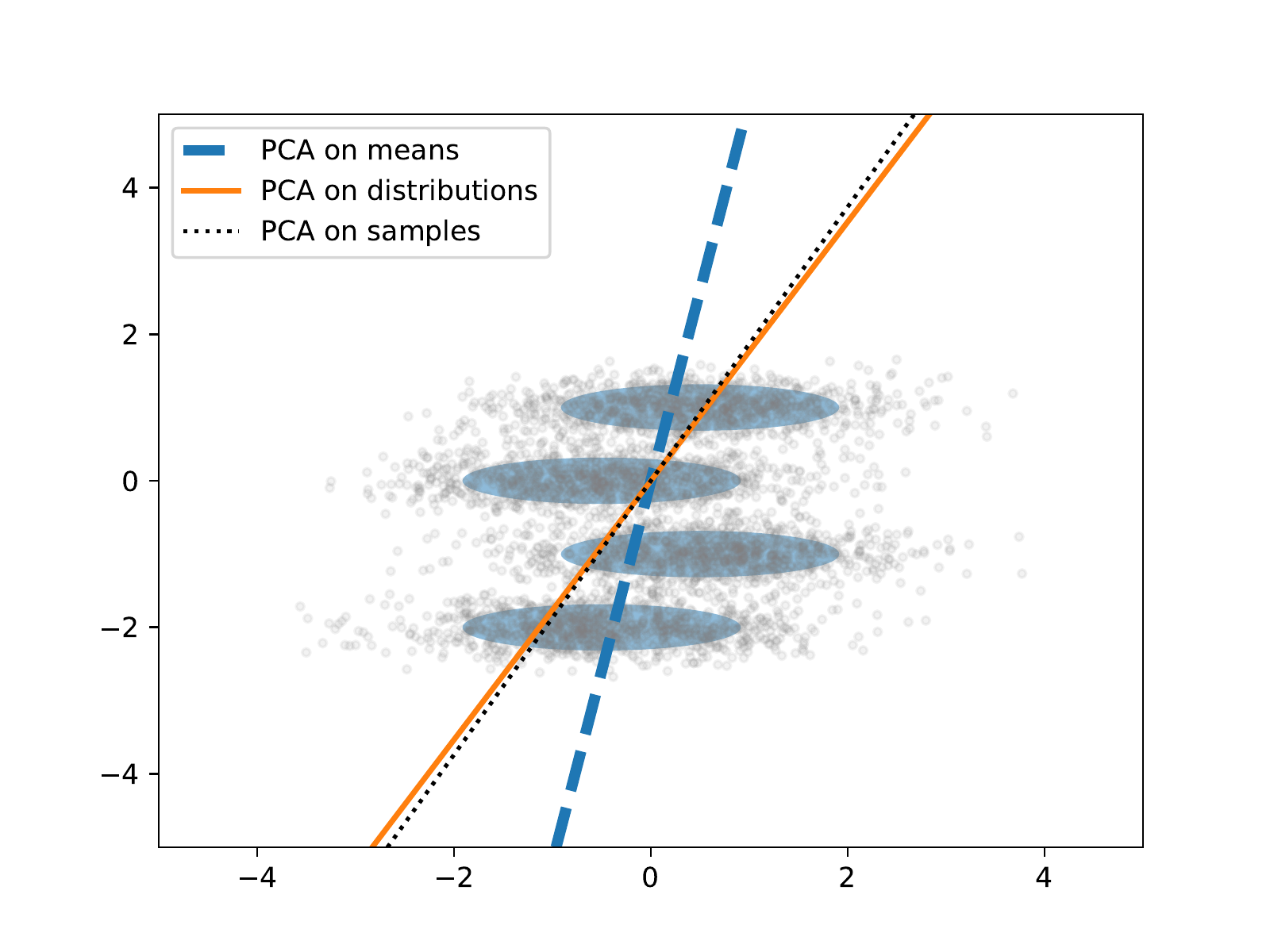}
\caption{\label{fig:samples} Comparing PCA on distributions, means, and samples,
on a dataset of four Gaussian distributions.}
\end{figure}

\paragraph{Related work.}
\citet{masarotto2022transportation} recently proposed a transportation-based PCA between covariance
matrices. Their formulation applies PCA in the tangent space of a manifold of
covariance operators and therefore leads to a different algorithm, somewhat more
computationally intensive as it requires calculating a Fr\'echet mean.
While our formulation only depends on covariance matrices through their sum,
their formulation seems more suited for capturing differences between individual
covariances. On the other hand, transportation PCA does not take into account
means, just covariances. We shall explore
the relationship and tradeoffs between the two formulations in the future.

\paragraph{Acknowledgements.}
This work is partly supported by NWO VI.Veni.212.228
and the European Union's Horizon Europe research and innovation programme
via UTTER 101070631. \hfill \euflag

\bibliographystyle{plainnat}
\bibliography{note.bib}

\end{document}